\documentclass[11pt]{article}
\pdfoutput=1

\usepackage{microtype}
\usepackage{graphicx}
\usepackage{subfigure}
\usepackage{booktabs} 
\usepackage{bm}
\usepackage{diagbox}
\usepackage{algorithm}
\usepackage{thm-restate}
\usepackage{natbib}
\usepackage[noend]{algorithmic}

\usepackage{geometry}
\usepackage{amssymb}
\usepackage{tikz}
\usepackage{url}
\usepackage{setspace}
\usepackage[pdftex,bookmarksnumbered,bookmarksopen,
colorlinks,citecolor=blue,linkcolor=blue,urlcolor=blue]{hyperref}
\usepackage{framed}
\usepackage{xcolor}
\usepackage{soul}
\usepackage{longtable}

\usepackage{times}
\usepackage{enumitem}
\usepackage{varwidth}
\usepackage{graphicx}
\usepackage{wrapfig}
\usepackage{enumerate}
\usepackage{caption}

\usepackage{amssymb}
\usepackage{bbm}
\usepackage{graphicx}
\usepackage{url}
\usepackage{setspace}
\usepackage{framed}
\usepackage{xcolor}
\usepackage{soul}
\usepackage{longtable}

\usepackage{times}
\usepackage{enumitem}
\usepackage{varwidth}
\usepackage{graphicx}
\usepackage{wrapfig}
\usepackage{enumerate}

\usepackage[utf8]{inputenc} 
\usepackage[T1]{fontenc}    
\usepackage{url}            
\usepackage{booktabs}       
\usepackage{amsfonts}       
\usepackage{nicefrac}       
\usepackage{microtype}      

\usepackage{graphicx}
\usepackage{appendix}
\usepackage{amsmath,amsfonts,amsthm}

\usepackage{mdwlist}
\usepackage{xspace}
\usepackage{color}
\usepackage{mathrsfs}

\usepackage{booktabs}
\usepackage{comment}

\usepackage{multirow}

\newcommand{\Appendix}[1]{the full version for}

\definecolor{maroon}{cmyk}{0,0.87,0.68,0.32}

\newtheorem{theorem}{Theorem}[section]
\newtheorem{lemma}[theorem]{Lemma}

\newtheorem{corollary}[theorem]{Corollary}
\newtheorem{remark}{Remark}

\newtheorem{definition}{Definition}

\newtheorem{example}{Example}

\newcommand{\R}{\mathbb{R}}

\renewcommand{\comment}[1]{}

\newcommand{\cA}{\mathcal{A}}

\newcommand{\cC}{\mathcal{C}}
\newcommand{\cD}{\mathcal{D}}
\newcommand{\cF}{\mathcal{F}}

\newcommand{\cL}{\mathcal{L}}

\newcommand{\cP}{\mathcal{P}}

\newcommand{\cX}{\mathcal{X}}

\newcommand{\bbE}{\mathbb{E}}

\definecolor{colorY}{rgb}{0.7 , 0.7 , 0.2}

\title{Gradient Descent with Provably Tuned Learning-rate Schedules}
\author{\textbf{Dravyansh Sharma}\\
Toyota Technological Institute at Chicago}

\begin{document}

\maketitle

\begin{abstract}%
Gradient-based iterative optimization methods are the workhorse of modern machine learning. They crucially rely on careful tuning of parameters like learning rate and momentum. However, one typically sets them using heuristic approaches without formal near-optimality guarantees. Recent work 
by Gupta and Roughgarden studies how to learn a good step-size in gradient descent. However, like most of the literature with theoretical guarantees for gradient-based optimization, their  results rely on strong assumptions on the function class including convexity and smoothness which do not hold in typical applications. In this work, we develop novel analytical tools for provably tuning hyperparameters in gradient-based algorithms that apply to non-convex and non-smooth functions. We obtain matching sample complexity bounds for learning the step-size in gradient descent shown for smooth, convex functions in prior work (up to logarithmic factors) but for a much broader class of functions. Our analysis applies to gradient descent on neural networks with commonly used activation functions (including ReLU, sigmoid and tanh). 
We extend our framework to tuning  multiple hyperparameters, including tuning the learning rate schedule, simultaneously tuning momentum and step-size, and pre-training the initialization vector. Our approach can be used to bound the sample complexity for minimizing both the validation loss as well as the number of gradient descent iterations.
\end{abstract}

\section{Introduction}

Gradient descent is a foundational optimization algorithm widely employed in machine learning and deep learning to minimize loss functions and improve model performance. A critical hyperparameter in gradient descent is the step size or learning rate, which dictates how far the algorithm moves along the negative gradient direction in each iteration. Selecting an appropriate step size is essential: a value too large can cause divergence, while one too small can lead to slow convergence or the inability to escape undesirable local minima leading to poor quality of the final iterate.

Considerable research has focused on tuning the step size for individual tasks. However, many real-world applications involve multi-task learning or repeated optimization across a collection of tasks. In such settings, the optimal step size may vary significantly between different task domains, and naive strategies such as using a fixed or globally tuned step size often yield suboptimal performance. This raises an important question: how can we effectively tune the step size of gradient descent across multiple, related tasks?

This paper explores principled approaches to tuning step size in gradient descent in multi-task environments. We investigate the theoretical underpinnings of step size sensitivity across tasks, by examining how the convergence varies as a function of the step size. We provide theoretical guarantees for the amount of data needed (sample complexity) for tuning the step size with provable guarantees, even in the presence of non-smooth and non-convex functions. Concretely,

\begin{itemize}[leftmargin=*,topsep=0pt,partopsep=1ex,parsep=1ex]\itemsep=-4pt
    \item We study tuning of the step-size (learning rate) in gradient descent across tasks in the framework of \cite{gupta2016pac}, but for a much broader class of functions. While prior work for sample complexity of step-size tuning assumes the class of optimized functions to be convex, smooth and satisfying a guaranteed progress assumption (roughly corresponds to strong convexity), our analysis works without any of these assumptions. We only assume that the class of functions is piecewise-polynomial, a property satisfied by neural networks with piecewise-polynomial activation functions (e.g., ReLU activation), or is piecewise-Pfaffian (which includes other commonly used activation functions like sigmoid and tanh).
    \item We show sample complexity bounds of $\tilde{O}(H^3/\epsilon^2)$ for uniform convergence for tuning the learning rate in the piecewise-polynomial case (asymptotically matching the previous bounds for convex and smooth functions, up to logarithmic terms), which implies generalization of the learned step-size to unseen instances (that is, small gap between training and test errors). 
    \item We further extend our techniques to obtain a $\tilde{O}(H^4/\epsilon^2)$ bound on the sample complexity of tuning the entire learning rate schedule.
    \item We show how to use our framework to tune the initialization scale parameter as well as the entire initialization vector. The latter corresponds to the sample complexity of pre-training using gradient descent.
    \item We show that our techniques extend beyond vanilla gradient descent. In particular, we show how to simultaneous tune the stepsize and momentum parameters in momemtum-based gradient descent (applicable to widely used optimizers like Adam).
    \item Finally, we show that our analytical techniques can be used to optimize not only the speed of convergence (which corresponds to the efficiency of the training procedure using gradient-based optimization), but also to optimize the quality of the last iterate (e.g.\ in terms of the validation loss of the learned network weights).
\end{itemize}

\subsection{Related work} 
Recently tuning the learning rate in gradient descent by meta learning has received significant research interest~\cite{Maclaurin2015GradientbasedHO,Andrychowicz2016LearningTL,Li2016LearningTO,Denevi2019LearningtoLearnSG,Khodak2019AdaptiveGM}. Despite empirical success,  theoretical understanding of setting a good learning rate is largely limited to convex and strongly-convex functions. In this work, we develop a principled understanding beyond the convex case under the recently introduced {\it data-driven algorithm design} paradigm~\citep{gupta2016pac,balcan2020data}. 

See Appendix \ref{app:related-work} for a more detailed treatment of related work.

\begin{algorithm}[t]
\caption{Gradient descent(step size $\eta$)}
\label{alg:gd-vanilla}
\textbf{Input}: Initial point $x$, function to minimize $f$, maximum number of iterations $H$, gradient threshold for convergence $\theta$  
\begin{algorithmic}[1]
\STATE Initialize $x_1\gets x$
\FOR{$i=1,\dots,H$}
\IF{$||\nabla f(x_i)||<\theta$}\label{line:gd-converge}
\STATE Return $x_i$
\ENDIF
\STATE $x_{i+1} = x_i - \eta \nabla f(x_i)$\label{line:gd-step}
\ENDFOR
\end{algorithmic}
\textbf{Output}: Return $x_{H+1}$
\end{algorithm}

\section{Notation and preliminaries}

We recall the notation and setup for data-driven tuning of gradient step introduced by \cite{gupta2016pac}.
The instance space of problems $\Pi$ here consists of pairs $(x,f)$  of initial points $x\in\R^d$ and functions $f:\R^d\rightarrow\R$. The family of gradient descent algorithms $\cA$ is given by $\cP\subset\R_+$ consisting of valid values of the step size $\eta$. We recall the vanilla gradient descent algorithm (Algorithm \ref{alg:gd-vanilla}).
We define the cost function  $\ell(\eta,x,f)$ for $\eta\in\cP$ and $(x,f)\in\Pi$ as the number of iterations for which Algorithm \ref{alg:gd-vanilla} runs on the instance $(x,f)$ when run with step size $\eta$. Let $\ell_{\eta}(x,f):=\ell(\eta,x,f)$ for all $\eta,x,f$. For guarantees on sample complexity of tuning $\eta$, we will be interested in the learning-theoretic complexity (pseudo-dimension) of the function class $\cL=\{\ell_{\eta}\mid \eta\in\cP\}$.

Notice that while prior work \cite{gupta2016pac} assumes strong sufficient conditions---convexity, $L$-smoothness, and a guaranteed progress condition (that is $\|x_{i+1}-x^*\|\le(1-c)\|x_{i}-x^*\|$ for some $c>0$ where $x^*$ is the unique stationary point and $x_i$ is the $i$th iterate in Algorithm \ref{alg:gd-vanilla}, this roughly corresponds to strong-convexity)---for the convergence to always happen and includes appropriate restrictions on the step-size $\eta$. In particular, these results do not apply to deep neural networks. Our results hold when all of these conditions are violated.  We handle non-convergence by assigning it the same cost as the maximum number of iterations, equal to $H$. Formally,

$$
\ell(\eta,x,f):=\begin{cases}
			\min\limits_{\substack{i\in[H]}} \|\nabla f(x_i)\|\le \theta, & \text{if such an $i$ exists,}\\
            H, & \text{otherwise.}
		 \end{cases}
$$

 \noindent Note that $x_i$ depends on $\eta,x,f$ and $i$. We consider the step-size as the hyperparameter of interest in this work, and assume other parameters like maximum number of iterations $H$ and gradient threshold for convergence $\theta$ are fixed and known.  We will use the $\Tilde{O}$ notation to suppress dependence on quantities apart from $H$ and the generalization error $\epsilon$ for simplicity.

\paragraph{Learning theory background.}
The pseudo-dimension is frequently used to analyze the learning theoretic complexity of real-valued  function classes, and will be a main tool in our sample complexity analysis. For completeness, we include below the formal definition.

\begin{definition}[Shattering and Pseudo-dimension, \cite{anthony1999neural}]\label{def:pdim}
Let $\cF$ be a set of functions mapping from $\cX$ to $\R$, and suppose that $S = \{x_1, \dots, x_m\} \subseteq \cX$. Then $S$ is pseudo-shattered by $\cF$ if there are real numbers $r_1, \dots, r_m$ such that for each $b \in \{0, 1\}^m$ there is a function $f_b$ in $\cF$ with $\mathrm{sign}(f_b(x_i) - r_i) = b_i$ for $i \in [m]$. We say that $r = (r_1, \dots, r_m)$ witnesses the shattering. We say that $\cF$ has pseudo-dimension $d$ if $d$ is the maximum cardinality of a subset $S$ of $\cX$ that is pseudo-shattered by $\cF$, denoted $\mathrm{Pdim}(\cF) = d$. If no such maximum exists, we say that $\cF$ has infinite pseudo-dimension. 
\end{definition}

\noindent Pseudo-dimension is a real-valued analogue of VC-dimension, and is a classic complexity notion in learning theory due to the following theorem which implies the uniform convergence 
for any function in class $\cF$ when $\mathrm{Pdim}(\cF)$ is finite. 
\begin{theorem}[$(\epsilon,\delta)$-uniform convergence sample complexity via pseudo-dimension, \cite{anthony1999neural}]\label{thm:pdim}
    Suppose $\cF$ is a class of real-valued functions with range in $[0, H]$ and finite $\mathrm{Pdim}(\cF)$. For every $\epsilon > 0$ and $\delta \in (0, 1)$, given any distribution $\cD$ over $\cX$, with probability $1-\delta$ over the draw of a sample $S\sim\cD^M$, for all functions $f\in\cF$, we have $|\frac{1}{M}\sum_{x\in S}f(x)-\bbE_{x\sim \cD}[f(x)]|\le \epsilon$ for some
    $M=O\left(\left(\frac{H}{\epsilon}\right)^2\left(\mathrm{Pdim}(\cF) + \log\frac{1}{\delta} \right)\right)$. 
\end{theorem}

\noindent 
We also need the following lemma from data-driven algorithm design, which bounds the pseudo-dimension of the class of loss functions, when the dual losses (i.e.\ losses as a function of some algorithmic hyperparameter computed on any fixed problem instance) have a piecewise constant structure with a bounded number of pieces.

\begin{lemma} (Lemma 2.3, \cite{balcan2020data})
     Suppose that for every problem instance $x \in \cX$, the function $u^*_x(\alpha) : \R\rightarrow\R$ is piecewise
constant with at most $N$ pieces. Then the family $\{u_\alpha(\cdot)\}$ over instances in $\cX$ has pseudo-dimension $O(\log N)$. \label{lm:piecewise-constant-to-pdim}
\end{lemma}

\section{Improved sample complexity guarantees for tuning the gradient descent step size}

\begin{table*}[t]
\centering
\begin{tabular}{cccc}
\hline
Parameter & Setting & Sample complexity bound & Result
\\ \hline
\hline
Stepsize&Convex+smooth+guaranteed progress&$\tilde{O}\left(\frac{H^3}{\epsilon^2}\right)$&\cite{gupta2016pac} \\\hline
Stepsize&Polynomial with degree $\Delta$&$\tilde{O}\left(\frac{H^3\log\Delta}{\epsilon^2}\right)$&Theorem \ref{thm:gd-poly-general-d} \\\hline
Stepsize&Piecewise-polynomial with degree $\Delta$&$\tilde{O}\left(\frac{H^3\log(pd\Delta)}{\epsilon^2}\right)$&Theorem \ref{thm:gd-pwpoly-general-d} \\
&and $p$ polynomial boundaries&& \\\hline
Stepsize&Pfaffian with chain length $q$, degree $\Delta$&$\tilde{O}\left(\frac{H^2(q^2d^2H^2+qdH\log(\Delta+M))}{\epsilon^2}\right)$&Theorem \ref{thm:gd-pfaffian-general-d} \\
&and Pfaffian degree $M$&& \\\hline
Learning-schedule& Piecewise-polynomial with degree&$\tilde{O}\left(\frac{H^4\log(pd\Delta)}{\epsilon^2}\right)$&Theorem \ref{thm:pwpoly-schedule} \\
& $\Delta$ and $p$ polynomial boundaries&& \\\hline
Initialization& Piecewise-polynomial with degree&$\tilde{O}\left(\frac{H^3\log(pd\Delta)}{\epsilon^2}\right)$&Theorem \ref{thm:pwpoly-initscale} \\
scale $\sigma$& $\Delta$ and $p$ polynomial boundaries&& \\\hline
Initialization& Piecewise-polynomial with degree&$\tilde{O}\left(\frac{dH^3\log(pd\Delta)}{\epsilon^2}\right)$&Theorem \ref{thm:pwpoly-init} \\
vector& $\Delta$ and $p$ polynomial boundaries&& \\\hline
Stepsize+& Piecewise-polynomial with degree&$\tilde{O}\left(\frac{H^3\log(pd\Delta)}{\epsilon^2}\right)$&Theorem \ref{thm:momentum-pwpoly-general-d} \\
momentum& $\Delta$ and $p$ polynomial boundaries&& \\\hline
Stepsize& Objective: Quality of final iterate; &$\tilde{O}\left(\frac{H^2(H\log (pd\Delta)+\log{\Delta_vp_v})}{\epsilon^2}\right)$&Theorem \ref{thm:validation} \\
& Piecewise-polynomial with degree&& \\
&$\Delta$ and $p$ polynomial boundaries&& \\

\hline
\end{tabular}
\caption{Summary of sample complexity bounds for tuning hyperparameters in gradient-based optimization. The objective is number of steps to convergence, unless stated otherwise. $d$ denotes the dimension of points $x$. 
}
\label{tab:comparison}
\end{table*}

Our overall approach towards bounding the pseudo-dimension of the cost function class $\cL=\{\ell_{\eta}:\Pi\rightarrow[0,H]\mid\eta\in\cP\}$ (which implies a bound on the sample complexity of tuning $\eta$ by classical learning theory), is by examining the structure of $\ell(\eta,x,f)$ on any fixed instance $(x,f)$ as the step-size $\eta$ is varied (also called the {\it dual} cost function $\ell_{x,f}(\eta)$ \cite{balcan2021much}) is very different from prior work. The key idea of prior work  \citep{gupta2016pac,jiao2025learning} is to establish a near-Lipschitzness property for the dual cost function, by bounding how far the number of steps to converge may diverge as the step-size is changed slightly, and then use a discretization argument over the space of step-sizes. It is easy to see that this approach cannot extend beyond very nicely-behaved functions (roughly, strongly-convex and $L$-smooth) as generally small changes to the step-size can cause dramatic changes to the number of the steps needed for convergence. In contrast, intuitively our analysis examines the {\it piecewise monotonicity} of the number of steps needed to converge as the step-size $\eta$ is varied.\looseness-1 

\subsection{Gradient descent for piecewise polynomial functions}

In this section, we will  consider several interesting function classes which are non-convex and non-smooth but intuitively have a bounded amount of oscillations. Our new analytical approach allows us to significantly extend the class of convex and smooth functions (with near strong-convexity properties) studied by \cite{gupta2016pac}, for which they obtain a $\tilde{O}(H^3/\epsilon^2)$ bound on the sample complexity of tuning the  step-size $\eta$ in Algorithm \ref{alg:gd-vanilla}. Remarkably, we will achieve the same asymptotic dependence on the sample complexity for the much broader class of functions using our new techniques.

We first consider the class of functions $f$ to be the class of polynomial functions in $d$ variables with a bounded degree $\Delta$. Our sample complexity bound only has a logarithmic dependence on $\Delta$, implying that they are meaningful even when the number of oscillations is exponentially large.
    Our overall approach is to show that location of the $i$-th iterate $x_i$ can be expressed as a bounded degree polynomial in $\eta$ using an inductive argument. This allows us to bound the number for intervals of $\eta$ for which Algorithm \ref{alg:gd-vanilla} converges in $i$ steps for any $1\le i\le H$. We can finally also bound the number of pieces of the dual cost function $\ell_{x,f}$ where gradient descent fails to converge. Formally, we have the following theorem (all omitted proofs are in Appendix \ref{app:proofs}).

\begin{restatable}{theorem}{thmgdpoly}\label{thm:gd-poly-general-d}
    Suppose the instance space $\Pi$ consists of $(x\in\R^d,f\in\cF)$, where  $\cF$ consists of polynomial functions $\R^d\rightarrow\R$ of degree at most $\Delta\ge 1$. Then $(\epsilon,\delta)$-uniform convergence is achieved for all step-sizes $\eta\in\cP\subset\R_{\ge 0}$ using $m = {O}\left(\frac{H^2}{\epsilon^2}(H\log\Delta+\log\frac{1}{\delta})\right)$ samples from $\cD$ for any distribution $\cD$ over $\Pi$.  
\end{restatable}

\begin{proof}
    We first claim that for $i\ge 2$, $x_i=(g_i^{(1)}(\eta),g_i^{(2)}(\eta),\dots,g_i^{(d)}(\eta))$ where $g_i^{(j)}$ is a polynomial function with degree at most $\Delta^{i-2}$. We will show this by induction.

    {\it Base case: $i=2$.} $x_2 = x_1 - \eta \nabla f(x_1) = x - \eta \nabla f(x)$ is a polynomial of degree $1=\Delta^{2-2}$ in $\eta$ in each coordinate. 

    {\it Inductive case: $i>2$.} Suppose $x_{i-1}=\boldsymbol{g}_{i-1}(\eta)=(g_{i-1}^{(j)}(\eta))_{j\in[d]}$ where $g_{i-1}^{(j)}$ is a polynomial of degree at most $\Delta^{i-3}$ (inductive hypothesis). Now $x_i=x_{i-1}-\eta \nabla f(x_{i-1})=\boldsymbol{g}_{i-1}(\eta)-\eta \nabla f(\boldsymbol{g}_{i-1}(\eta))=:\boldsymbol{g}_i(\eta)$. Clearly, $\boldsymbol{g}_i^{(j)}$ is a polynomial in $\eta$, with degree at most $\Delta^{i-3}(\Delta-1)+1 = \Delta^{i-2} -\Delta^{i-3}+1\le \Delta^{i-2}$.

    Thus, for any fixed $1\le i\le H$, $x_i=\boldsymbol{g}_i(\eta)$ where each coordinate $g_{i}^{(j)}$ is a polynomial with degree at most $\Delta^{i-2}$. For a fixed initial point $x$, this implies that $\nabla f(x_i)$ is a polynomial in $\eta$ of degree at most $\Delta^{i-1}$ in each coordinate. Thus, $\|\nabla f(x_i)\|^2$ is a polynomial of degree at most $2\Delta^{i-1}$ in $\eta$. Now, consider the set of points $\eta$ for which $\|\nabla f(x_i)\|^2<\theta^2$ for some constant $\theta$. This consists of at most $O(\Delta^{i-1})$ intervals. 

    Furthermore, we note that for any $\eta$, the cost is determined by the smallest $i$ such that $\|\nabla f(x_i)\|<\theta$ (if one exists). Since the cost takes only discrete values, it is a piecewise-constant function of $\eta$, and we seek to bound the number of these pieces. A naive counting argument gives a $O(\Pi_{i=1}^H\Delta^{i-1})=O(\Delta^{H^2})$ bound on the number of pieces, since if there are $K_{i-1}$ intervals corresponds to values of $\eta$ for which the algorithm converges within $i-1$ steps, then each of the $O(\Delta^{i-1})$ intervals  in round $i$ computed above may result in at most $K_{i-1}+1$ new pieces. We can, however, use an amortized counting argument to give a tighter bound. Indeed, suppose there are $K_{i-1}$ intervals with different values of cost $\le i-1$.  
    Of the new  $O(\Delta^{i-1})$ intervals say $T_i$ intersect at least one of the existing  pieces, resulting in at most  $T_i+K_{i-1}$ new pieces overall. 
    Thus, the total number of pieces of the cost function with cost $\le i$ is 
    $K_i\le (T_i+K_{i-1})+O(\Delta^{i-1})-T_i+K_{i-1}\le O(\Delta^{i-1})+2K_{i-1}$. Thus, across $i$,  
    we have at most $O(\sum_{i=1}^H2^{H-i}\Delta^{i-1})=O(\Delta^{H})$ intervals over each of which Algorithm \ref{alg:gd-vanilla} converges with a constant number of steps. The cost over the remainder of the domain $\cP$ where the algorithm does not converge, which consists of $O(\Delta^{H})$ intervals, is $H$.

    Thus, using Lemma \ref{lm:piecewise-constant-to-pdim}, since each function $\ell_{x,f}$ is $O(\Delta^{H})$-monotonic,  we get a bound on the pseudo-dimension of $\cL$ of $O(H\log\Delta)$, which implies the stated sample complexity.
\end{proof}

\noindent The uniform convergence guarantee in the above result is quite strong. It implies that we only need to find an approximately optimal $\eta$ on the training set, and our guarantees will imply a small generalization error.
We further extend the result to piecewise polynomial functions with polynomial boundaries. For simplicity, we assume that the gradient descent procedure never lands exactly on any boundary (necessary for Algorithm \ref{alg:gd-vanilla} to be well-defined). 

\begin{restatable}{theorem}{thmgdpwpoly}\label{thm:gd-pwpoly-general-d}
    Suppose the instance space $\Pi$ consists of $(x\in\R^d,f\in\cF)$, where  $\cF$ consists of functions $\R^d\rightarrow\R$ that are piecewise-polynomial  with at most $p$ polynomial boundaries, with the maximum degree of any piece function or boundary function  at most $\Delta\ge 1$. Then $(\epsilon,\delta)$-uniform convergence is achieved for all step-sizes $\eta\in\cP\subset\R_{\ge 0}$ using $m = {O}\left(\frac{H^2}{\epsilon^2}(H\log (pd\Delta)+\log\frac{1}{\delta})\right)$ samples from $\cD$ for any distribution $\cD$ over $\Pi$.
\end{restatable}

\begin{proof}
    We first claim that for $i\ge 2$, $x_i=(g_i^{(1)}(\eta),g_i^{(2)}(\eta),\dots,g_i^{(d)}(\eta))$ where each $g_i^{(j)}$ is a piecewise-polynomial function with degree at most $\Delta^{i-2}$ and at most $(2pd\Delta)^{i-2}$ pieces. We will show this by induction.

    {\it Base case: $i=2$.} $x_2 = x_1 - \eta \nabla f(x_1) = x - \eta \nabla f(x)$ is a polynomial of degree $1=\Delta^{2-2}$ in $\eta$ in each coordinate. 

    {\it Inductive case: $i>2$.} Suppose $x_{i-1}=\boldsymbol{g}_{i-1}(\eta)=(g_{i-1}^{(j)}(\eta))_{j\in[d]}$ where $g_{i-1}^{(j)}$ is piecewise-polynomial with at most $(2pd\Delta)^{i-3}$ pieces and degree at most $\Delta^{i-3}$ (inductive hypothesis). Now $x_i=x_{i-1}-\eta \nabla f(x_{i-1})=\boldsymbol{g}_{i-1}(\eta)-\eta \nabla f(\boldsymbol{g}_{i-1}(\eta))=:\boldsymbol{g}_i(\eta)$. Now, $\nabla f(\boldsymbol{g}_{i-1}(\eta))$ is piecewise-polynomial. The degree is at most $(\Delta-1)\Delta^{i-3}\le \Delta^{i-2}-1$. Any critical point is either a critical point of some $g_{i-1}^{(j)}$, or a solution of the equation $f^{(k)}(g_{i-1}^{(1)}(\eta),\dots,g_{i-1}^{(d)}(\eta))=0$ for some boundary function $f^{(k)}$ ($k\in[p]$) of $f$. The total number of critical points of $\boldsymbol{g}_i^{(j)}$ (for any $j\in[d]$) is therefore at most $(2pd\Delta)^{i-3}+(pd\Delta)(2pd\Delta)^{i-3}\le (2pd\Delta)^{i-2}$.
Therefore, $\boldsymbol{g}_i^{(j)}$ is piecewise-polynomial in $\eta$, with degree at most $\Delta^{i-2}$ and at most $(2pd\Delta)^{i-2}$ pieces.

    Thus, for any fixed $1\le i\le H$, $x_i=\boldsymbol{g}_i(\eta)$ where each coordinate $g_{i}^{(j)}$ is piecewise-polynomial with degree at most $\Delta^{i-2}$ and at most $(2pd\Delta)^{i-2}$ pieces. Using the argument above, for any fixed initial point $x$, $\|\nabla f(x_i)\|^2$ is piecewise-polynomial with degree at most $2\Delta^{i-1}$ in $\eta$ and at most $(2pd\Delta)^{i-1}$ pieces. Now, consider the set of points $\eta$ for which $\|\nabla f(x_i)\|^2<\theta^2$ for some constant $\theta$. This consists of at most $O((2pd\Delta)^{i-1})$ intervals. 
    
We can now apply the amortized counting argument from the proof of Theorem \ref{thm:gd-poly-general-d} to conclude that the number of pieces in the dual cost function is $O((2pd\Delta)^{H})$ here.

    Thus, using Lemma \ref{lm:piecewise-constant-to-pdim}, we get a bound on the pseudo-dimension of $\cL$ of $O(H\log pd\Delta)$, which implies the stated sample complexity.
\end{proof}

\noindent The case of piecewise-polynomial functions is particularly interesting. As discussed below it captures gradient descent for an important class of feedforward neural networks. 

\begin{example}
    For deep neural networks with piecewise polynomial activation functions, the network computes a piecewise polynomial function of its weights on any fixed input $x$ \cite{bartlett1998almost}. Therefore the MSE loss of the network on a given dataset is also a piecewise polynomial function of its weights. Therefore our results for tuning the gradient descent step-size above apply in this case. Concretely, Theorem \ref{thm:gd-pwpoly-general-d} implies a sample complexity bound of $\tilde{O}\left(\frac{H^3L\log k}{\epsilon^2}\right)$, where $L$ is the number of layers in the network and $k$ is the number of nodes (using \cite{bartlett1998almost}). The piecewise-polynomial structure also holds if we add regularization terms related to flatness of the minima e.g.\ $\|\nabla^2 f\|$ to the loss function. 
\end{example}

\noindent We note that our techniques can be used to establish  sample complexity bounds for tuning the entire learning rate schedule (see Theorem \ref{thm:pwpoly-schedule}).

\subsection{Beyond polynomial functions}
We extend our techniques beyond  polynomial functions to a much broader class of Pfaffian functions. This will allow us to give guarantees for gradient descent for learning neural networks with non-polynomial activation functions including sigmoid and tanh. We will employ tools used for the analysis of Pfaffian functions~\citep{karpinski1997polynomial,balcan2024algorithm}.

Intuitively, Pfaffian functions are functions for which the derivatives can be expressed as a polynomial of the variables and the function itself (or other functions from a sequence of functions, as explained below).
Formally, we first need the notion of a \textit{Pfaffian chain}. Roughly speaking, it consists of an ordered sequence of functions, in which the derivative of each function can be represented as a polynomial of the variables and previous functions in the sequence (including the function itself). 
\begin{definition}[Pfaffian Chain, {\citealp{khovanskiui1991fewnomials}}] \label{def:pfaffian-chain}
    A finite sequence of continuously differentiable functions $\eta_1, \dots, \eta_q: \R^d \rightarrow \R$ and variables $\boldsymbol{a} = (a_1, \dots, a_d) \in \R^d$ form a Pfaffian chain $\cC(\boldsymbol{a}, \eta_1, \dots, \eta_q)$ if there are real polynomials $P_{i, j}(\boldsymbol{a}, \eta_1, \dots, \eta_j)$ in $a_1, \dots, a_d, \eta_1, \dots, \eta_j$, for  all $i \in [d]$ and $j \in [q]$, such that
    \begin{equation*}
        \begin{aligned} 
            \frac{\partial \eta_j}{\partial a_i} = P_{i, j}(\boldsymbol{a}, \eta_1, \dots, \eta_j).
        \end{aligned}
    \end{equation*}
\end{definition}

\noindent {Note that $P_{i, j}(\boldsymbol{a}, \eta_1, \dots, \eta_j)$ is a polynomial in $\boldsymbol{a}$ and the functions $\eta_1(\boldsymbol{a}), \dots, \eta_j(\boldsymbol{a})$.  Pfaffian chains are associated with a \textit{length} and a \textit{Pfaffian degree}, that dictate their complexity. The length of a Pfaffian chain is the number of functions $q$ that appear on that chain, while the Pfaffian degree of a chain is the maximum degree of polynomials $P_{i,j}$ that can be used to express the partial derivative of functions on that chain.

\noindent Given a Pfaffian chain, one can define the Pfaffian function as  a polynomial of variables and functions on that chain. 
\begin{definition}[Pfaffian functions, {\citealp{khovanskiui1991fewnomials}}]
    \label{def:pfaffian-function}
      Given a Pfaffian chain $\cC(\boldsymbol{a}, \eta_1, \dots, \eta_q)$, as defined in Definition \ref{def:pfaffian-chain}, a Pfaffian function over the chain $\cC$ is a function of the form $g(\boldsymbol{a}) = Q(\boldsymbol{a}, \eta_1, \dots, \eta_q)$, where $Q$ is a polynomial in variables $\boldsymbol{a}$ and functions $\eta_1, \dots, \eta_q$ in the chain $\cC$. { The degree $\Delta$ of the Pfaffian function $g(\boldsymbol{a})$ is the degree of the polynomial $Q(\boldsymbol{a}, \eta_1, \dots, \eta_q)$.}
\end{definition}

\noindent For example, $\frac{d e^x}{dx}=e^x$, and therefore the function $g(x)=e^x$ is a Pfaffian function of degree $\Delta=1$, associated with a Pfaffian chain of length $1$ and Pfaffian degree $1$. Note that polynomial functions are Pfaffian with degree equal to their usual degree and associated with a chain of length zero. We show the following generalization of Theorem \ref{thm:gd-poly-general-d}.

\begin{restatable}{theorem}{thmgdpfaffian}\label{thm:gd-pfaffian-general-d}
    Suppose the instance space $\Pi$ consists of $(x\in\R^d,f\in\cF)$, where  $\cF$ consists of Pfaffian functions $\R^d\rightarrow\R$ of degree at most $\Delta\ge 1$, associated with Pfaffian chains of length at most $q$ and Pfaffian degree at most $M$. Then $(\epsilon,\delta)$-uniform convergence is achieved for all step-sizes $\eta\in\cP\subset\R_{\ge 0}$ using $m = {O}\left(\frac{H^2}{\epsilon^2}(q^2d^2H^2+qdH\log(\Delta+M)+\log\frac{1}{\delta})\right)$ samples from $\cD$ for any distribution $\cD$ over $\Pi$.  
\end{restatable}

\begin{proof}
    Fix a function $f\in\cF$ and initial point $x$. Suppose $f(x)=Q(x, f_1(x), \dots, f_q(x))$ is associated with the Pfaffian chain $\cC(x, f_1, \dots, f_q)$ of length at most $q$, with degree of $Q$ at most $\Delta$. Note that, by the definition of Pfaffian functions, each coordinate of $\nabla f(x)$ is a polynomial in $x, f_1, \dots, f_q$ of degree at most $\Delta+M-1$.
    
    We first claim that for $i\ge 2$, $x_i=(g_i^{(1)}(\eta),g_i^{(2)}(\eta),\dots,g_i^{(d)}(\eta))$ where $g_i^{(j)}$ is a Pfaffian function with degree at most $(\Delta+M)^{i-2}$, and associated with the Pfaffian chain \begin{align*}
        \cC_i:=\cC(\eta, f_1, \dots, f_q, (g_2^{(j)})_{j\in[d]}, (f_k(g_2^{(j)}))_{k\in[q],j\in[d]},\dots,\\(g_{i-1}^{(j)})_{j\in[d]}, (f_k(g_{i-1}^{(j)}))_{k\in[q],j\in[d]})
    \end{align*} of length at most $q+(q+1)d(i-2)$ as the function $f$. We will show this by induction.

    {\it Base case: $i=2$.} $x_2 = x_1 - \eta \nabla f(x_1) = x - \eta \nabla f(x)$ is a polynomial of degree $1=(\Delta+M)^{2-2}$ in $\eta$ in each coordinate. 

    {\it Inductive case: $i>2$.} Suppose $x_{i-1}=\boldsymbol{g}_{i-1}(\eta)=(g_{i-1}^{(j)}(\eta))_{j\in[d]}$ where $g_{i-1}^{(j)}$ is a Pfaffian function of degree at most $(\Delta+M)^{i-3}$ and associated with the chain $\cC_{i-1}$ (inductive hypothesis). 
    
    Now $x_i=x_{i-1}-\eta \nabla f(x_{i-1})=\boldsymbol{g}_{i-1}(\eta)-\eta \nabla f(\boldsymbol{g}_{i-1}(\eta))=:\boldsymbol{g}_i(\eta)$. By the definition of Pfaffian functions, $(\nabla f(\boldsymbol{g}_{i-1}(\eta)))_j$ is a polynomial in $\boldsymbol{g}_{i-1}(\eta),f_1(\boldsymbol{g}_{i-1}(\eta)),\dots,f_q(\boldsymbol{g}_{i-1}(\eta))$ with degree at most $\Delta+M-1$, and therefore $x_i$ 
    has degree at most $(\Delta+M)^{i-3}(\Delta+M-1)+1 \le (\Delta+M)^{i-2}$ in the chain $\cC_i$.

    Thus, for any fixed $1\le i\le H$, $x_i=\boldsymbol{g}_i(\eta)$ where each coordinate $g_{i}^{(j)}$ is a Pfaffian function with degree at most $(\Delta+M)^{i-2}$, associated with the chain $\cC_i$. For a fixed initial point $x$, this implies that  $\|\nabla f(x_i)\|^2$ is a Pfaffian function of degree at most $2(\Delta+M)^{i-1}$ in the chain $\cC_i$ of length $O(qdi)$. Now, consider the set of points $\eta$ for which $\|\nabla f(x_i)\|^2<\theta^2$ for some constant $\theta$. By a standard connected components bound for Pfaffian functions (e.g. see Corollary B.3, \citep{balcan2024algorithm}), this consists of at most $2^{O(q^2d^2i^2)}\cdot(\Delta+M)^{O(qdi)}$ intervals. 
    
    By the counting argument in the proof of Theorem \ref{thm:gd-poly-general-d}, across $i\in[H]$,  
    we have $2^{O(q^2d^2H^2)}(\Delta+M)^{O(qdH)}$ intervals over each of which Algorithm \ref{alg:gd-vanilla} converges with a constant number of steps (or does not converge and the cost is $H$). 
    Thus, using Lemma \ref{lm:piecewise-constant-to-pdim},   we get a bound on the pseudo-dimension of $\cL$ of $O(q^2d^2H^2+qdH\log(\Delta+M))$, which implies the stated sample complexity.
\end{proof}

\noindent A key technical insight in the above proof is the careful construction of the Pfaffian chain $\cC_i$. A naive construction would include all function  composition sequences $f_{i_1}\circ f_{i_2}\circ\dots\circ f_{i_n}$ in the chain, where $f_{i_j}$ are functions from the chain $f_1, \dots, f_q$ and $n\in[H]$, resulting in an exponential (in $q$) upper-bound on the sample complexity. We note that the above result applies to tuning the learning rate in gradient descent for neural networks with other commonly used activations, including the sigmoid activation $\sigma(x)=1 / (1 + e^{-x})$, and the tanh activation $\tau(x)=(e^x - e^{-x}) / (e^x + e^{-x})$.

\begin{remark}
    For deep networks with sigmoid (or tanh) activation functions, the network computes a piecewise Pfaffian function of its weights on any fixed input $x$ \cite{karpinski1997polynomial}. The MSE loss of the network on a given dataset is also a piecewise Pfaffian function of its weights. Our results for tuning the gradient descent step-size above apply in this case. Concretely, Theorem \ref{thm:gd-pwpoly-general-d} implies a sample complexity bound of $\tilde{O}\left(\frac{H^4k^4}{\epsilon^2}\right)$, where $k$ is the number of nodes (using results from \cite{karpinski1997polynomial}).
\end{remark}

We conclude this section with some more examples and some non-examples of Pfaffian functions. All polynomials and ratios of polynomials are Pfaffian, as well as polynomials in the exponential and logarithm function. Functions involving fractional exponents (e.g.\ $f(x)=x^{1/2}+x^{1/3}$) are also Pfaffian. However, period trigonometric functions like sine and cosine are not Pfaffian (although non-periodic inverse trigonometric functions, like arctan are).

\section{Tuning multiple hyperparameters in gradient-based optimization}
We will now extend our framework from tuning a single stepsize parameter to multiple parameters. First, we bound the sample complexity of tuning the entire learning rate schedule from data. For piecewise-polynomial functions, we show that a slightly larger $\tilde{O}(H^4/\epsilon^2)$ bound on the number of samples is sufficient to learn the entire step-size schedule for gradient descent. Next, we show the generality of our framework by showing that it can be used to learn how to initialize gradient descent. In the context of using gradient descent for tuning the weights of neural networks, this implies a bound on the sample complexity of effective pre-training (we focus on efficient convergence here but our framework extends to optimizing the quality of the final iterate, see Section \ref{sec:quality}). Finally, we show that our framework can be used to analyze the tuning of hyperparameters beyond vanilla gradient descent, by showing how to simultaneously tune the learning rate and momentum hyperparameters.
\subsection{Learning the step-size schedule}
Designing a good learning rate schedule is considered a crucial problem in gradient-based iterative optimization. Our framework allows learning an iterate-dependent learning rate. That is, we set the learning rate to $\eta_i$ for $i\in[H]$, and we learn the sequence $\eta_i$ from data.

\begin{restatable}{theorem}{thmpwpolyschedule}\label{thm:pwpoly-schedule}
    Consider a variant of Algorithm \ref{alg:gd-vanilla}, where a different step-size $\eta_i$ is used in the $i$-th update, i.e., $x_{i+1}=x_i-\eta_i\nabla f(x_i)$ with parameters $\eta_i$ chosen from some continuous set $\mathcal{P}\subset\R_{\ge 0}^H$. Suppose the instance space $\Pi$ consists of $(x\in\R^d,f\in\cF)$, where  $\cF$ consists of functions $\R^d\rightarrow\R$ that are piecewise-polynomial  with at most $p$ polynomial boundaries, with the maximum degree of any piece function or boundary function  at most $\Delta\ge 1$. Then $(\epsilon,\delta)$-uniform convergence is achieved for all  $(\eta_1,\dots,\eta_H)\in\cP$ using $m = {O}(\frac{H^2}{\epsilon^2}(H^2\log(pd\Delta)+\log\frac{1}{\delta}))$ samples from $\cD$ for any distribution $\cD$ over $\Pi$.
\end{restatable}

\begin{proof}
    We first claim that for $i\ge 2$, $x_i=(g_i^{(1)}(\eta_1,\dots,\eta_{i-1}),\dots,g_i^{(d)}(\eta_1,\dots,\eta_{i-1}))$ where each $g_i^{(j)}$ is a piecewise-polynomial function with degree at most $\Delta^{i-2}$ and at most $(2pd)^{i-2}$ algebraic boundaries of degree at most $\Delta^{i-2}$. We  show this by induction.

    {\it Base case: $i=2$.} $x_2 = x_1 - \eta_1 \nabla f(x_1) = x - \eta_1 \nabla f(x)$ is a polynomial of degree $1=\Delta^{2-2}$ in $\eta_1$ in each coordinate. 

    {\it Inductive case: $i>2$.} Suppose $x_{i-1}=\boldsymbol{g}_{i-1}(\boldsymbol{\eta}_{i-2})=(g_{i-1}^{(j)}(\eta_1,\dots,\eta_{i-2}))_{j\in[d]}$ where $g_{i-1}^{(j)}$ is piecewise-polynomial with at most $(2pd)^{i-3}$ boundaries and degree at most $\Delta^{i-3}$ (for both pieces and boundaries, by the inductive hypothesis). 
    
    Now $x_i=x_{i-1}-\eta_{i-1} \nabla f(x_{i-1})=\boldsymbol{g}_{i-1}(\boldsymbol{\eta}_{i-2})-\eta_{i-1} \nabla f(\boldsymbol{g}_{i-1}(\boldsymbol{\eta}_{i-2}))=:\boldsymbol{g}_i(\boldsymbol{\eta}_{i-1})$. Now, $\nabla f(\boldsymbol{g}_{i-1}(\boldsymbol{\eta}_{i-2}))$ is piecewise-polynomial. The degree of  is at most $(\Delta-1)\Delta^{i-3}\le \Delta^{i-2}-1$. Any discontinuity point is either a discontinuity point of some $g_{i-1}^{(j)}$, or a solution of the equation $f^{(k)}(g_{i-1}^{(1)}(\boldsymbol{\eta}_{i-2}),\dots,g_{i-1}^{(d)}(\boldsymbol{\eta}_{i-2}))=0$ for some boundary function $f^{(k)}$ ($k\in[p]$) of $f$. The total number of boundary functions of $\boldsymbol{g}_i^{(j)}$ (for any $j\in[d]$) is therefore at most $d(2pd)^{i-3}+(pd)(2pd)^{i-3}\le (2pd)^{i-2}$.
Therefore, $\boldsymbol{g}_i^{(j)}$ is piecewise-polynomial in $\boldsymbol{\eta}_{i-1}$, with degree at most $\Delta^{i-2}$ and at most $(2pd\Delta)^{i-2}$ pieces.

    Thus, for any fixed $1\le i\le H$, $x_i=\boldsymbol{g}_i(\boldsymbol{\eta}_{i-1})$ where each coordinate $g_{i}^{(j)}$ is piecewise-polynomial with degree at most $\Delta^{i-2}$ and at most $(2pd)^{i-2}$ algebraic boundaries. Using the argument above, for any fixed initial point $x$, $\|\nabla f(x_i)\|^2$ is piecewise-polynomial with degree at most $2\Delta^{i-1}$ in $\boldsymbol{\eta}_{i-1}$ and at most $(2pd\Delta)^{i-1}$ boundaries. This implies that the dual cost function can be computed using a GJ algorithm with degree at most $O(\Delta^H)$ and predicate complexity $O(\sum_i(2pd)^i)=(2pd)^{O(H)}$. Using~\cite{bartlett2022generalization}, we get a bound on the pseudo-dimension of $\cL$ of $O(H^2\log pd\Delta)$, which implies the stated sample complexity.
\end{proof}

\subsection{Learning to initialize}
Another crucial parameter that impacts the convergence of gradient descent is the initial point $x$. For example, one may set the initial weights of a neural network by drawing each coordinate according to the Gaussian distribution $\mathcal{N}(0,\sigma^2)$, where $\sigma$ is an important hyperparameter that must be carefully set~\citep{lecun1998efficient,glorot2010understanding,he2015delving}. Here, we will treat the norm of the initial point $x$ as the {\it initialization scale} parameter, i.e.\ $x=\sigma \hat{x}$, where $\|\hat{x}\|=1$ and $\sigma>0$.

\begin{restatable}{theorem}{thmpwpolyinitscale}\label{thm:pwpoly-initscale}
    Consider a variant of Algorithm \ref{alg:gd-vanilla}, where the learning rate $\eta$ is fixed, but the initial point is given by $x=\sigma \hat{x}$ for some fixed $\hat{x}$ and hyperparameter $\sigma\in\R_{>0}$. Suppose the instance space $\Pi$ consists of $(\hat{x}\in\R^d,f\in\cF)$, where  $\cF$ consists of functions $\R^d\rightarrow\R$ that are piecewise-polynomial  with at most $p$ polynomial boundaries, with the maximum degree of any piece function or boundary function  at most $\Delta\ge 1$. Then $(\epsilon,\delta)$-uniform convergence is achieved for all  $\sigma$ using $m = {O}(\frac{H^2}{\epsilon^2}(H\log(pd\Delta)+\log\frac{1}{\delta}))$ samples from $\cD$ for any distribution $\cD$ over $\Pi$.
\end{restatable}

\noindent We can also learn the entire initialization vector $x$, with an extra factor of $d$ in the sample complexity. This is remarkable, as in the context of neural networks this corresponds to learning a pre-trained weight for a collection of tasks (given by the task distribution $\cD$).

\begin{restatable}{theorem}{thmpwpolyinit}\label{thm:pwpoly-init}
    Consider a variant of Algorithm \ref{alg:gd-vanilla}, where the learning rate $\eta$ is fixed, but the initial point is given by $x\in\R^d$ (a hyperparameter). Suppose the instance space $\Pi$ consists of $(f\in\cF)$, where  $\cF$ consists of functions $\R^d\rightarrow\R$ that are piecewise-polynomial  with at most $p$ polynomial boundaries, with the maximum degree of any piece function or boundary function  at most $\Delta\ge 1$. Then $(\epsilon,\delta)$-uniform convergence is achieved for all  $x\in\R^d$ using $m = {O}(\frac{H^2}{\epsilon^2}\left(dH\log(pd\Delta)+\log\frac{1}{\delta})\right)$ samples from $\cD$ for any distribution $\cD$ over $\Pi$.
\end{restatable}

\subsection{Beyond vanilla gradient descent}

We will now show that our technique extends beyond tuning the learning rate in gradient descent to tuning relevant hyperparameters in other popular iterative gradient based optimization methods, showing the versatile applicability of our analytical framework. In particular, we will show how to tune the momemtum and learning rate parameters $\gamma,\eta$ simultaneously in Algorithm \ref{alg:momentum}. Momentum~\cite{Rumelhart1986LearningRB} takes an  exponentially weighted average of the gradients to update the points at each iteration, and is particularly important for optimizing non-convex functions. It is widely used in practice and is a part of optimizers like Adam. We extend our approach to show how to tune the momentum parameter and the learning rate in momentum-based gradient descent.

\begin{restatable}{theorem}{thmmmpwpoly}\label{thm:momentum-pwpoly-general-d}
    Consider the problem of tuning the $\eta,\gamma$ in Algorithm \ref{alg:momentum} over some continuous set $\mathcal{P}\subset\R_{\ge 0}^2$. Suppose the instance space $\Pi$ consists of $(x\in\R^d,f\in\cF)$, where  $\cF$ consists of functions $\R^d\rightarrow\R$ that are piecewise-polynomial  with at most $p$ polynomial boundaries, with the maximum degree of any piece function or boundary function  at most $\Delta\ge 1$. Then $(\epsilon,\delta)$-uniform convergence is achieved for all  $\gamma,\eta\in\cP$ using $m = {O}\left(\frac{H^2}{\epsilon^2}(H\log(pd\Delta)+\log\frac{1}{\delta})\right)$ samples from $\cD$ for any distribution $\cD$ over $\Pi$.
\end{restatable}

\begin{algorithm}[t]
\caption{Momentum-based gradient descent(step size $\eta$, momentum parameter $\gamma$)}
\label{alg:momentum}
\textbf{Input}: Initial point $x$, function to minimize $f$, maximum number of iterations $H$, gradient threshold for convergence $\theta$  
\begin{algorithmic}[1]
\STATE Initialize $x_1\gets x, y_1\gets 0$
\FOR{$i=1,\dots,H$}
\IF{$||\nabla f(x_i)||<\theta$}\label{line:gd-converge}
\STATE Return $x_i$
\ENDIF
\STATE $y_{i+1} = \gamma y_i-\eta \nabla f(x_i)$
\STATE $x_{i+1} = x_i +y_i$\label{line:mm-step}
\ENDFOR
\end{algorithmic}
\textbf{Output}: Return $x_{H+1}$
\end{algorithm}

\noindent Note  the logarithmic dependence on the degree and dimensionality in our bounds. This makes our bounds meaningful for tuning networks with a large number weights $d$ and a large number of layers $L$ ($\Delta$ is typically exponential in $L$, so our bounds imply a linear dependence on $L$).

\section{Beyond convergence rates}\label{sec:quality}

Let $\tilde{x}$ denote the final point output by Algorithm \ref{alg:gd-vanilla}. The quality of optimum learned is given by some function $g(\tilde{x})$, which may be different from the function $f$ on which gradient descent was performed. 

Formally, a problem instance is given by a tuple $(x, f, f_v)$ consisting of an initial point $x\in\R^d$, and $f, f_v:\R^d\rightarrow\R$ denoting the {\it training} and {\it validation} functions respectively. Gradient descent (Algorithm \ref{alg:gd-vanilla}) is performed using $x$ as the initial point, with $f$ as the function to minimize and with access to $\nabla f$. The output of gradient descent on the training function, $\tilde{x}$, is evaluated using the validation function ${f_v}$. For example, $f_v(x)$ could denote the loss of the  neural network with  weights $x$ on the validation set.

\begin{restatable}{theorem}{thmvalidation}\label{thm:validation}
Suppose the instance space $\Pi$ consists of $(x\in\R^d,f\in\cF,f_v\in\cF_v)$, where  $\cF$ (resp.\ $\cF_v)$ consists of functions $\R^d\rightarrow\R$ that are piecewise-polynomial  with at most $p$ (resp.\ $p_v$) polynomial boundaries, with the maximum degree of any piece function or boundary function  at most $\Delta\ge 1$ (resp.\ $\Delta_v$). Then $(\epsilon,\delta)$-uniform convergence w.r.t.\ the validation loss is achieved for all step-sizes $\eta\in\cP\subset\R_{\ge 0}$ using $m = {O}\left(\frac{H^2}{\epsilon^2}(H\log (pd\Delta)+\log\frac{\Delta_vp_v}{\delta})\right)$ samples from $\cD$ for any distribution $\cD$ over $\Pi$.    
\end{restatable}

\section{Conclusion}
We develop a new framework for analyzing the sample complexity of tuning the hyperparameters in gradient-based optimization. While prior theoretical techniques are largely limited in their scope to convex and smooth functions, we significantly expand the set of functions for which one can tune the stepsize of gradient descent including piecewise-polynomial and piecewise-Pfaffian functions. Our results imply finite (polynomial) bounds on the sample complexity of tuning the stepsize and stepsize schedules, and also apply to learning a pre-trained network using gradient descent.

\section*{Acknowledgments}

This work involved extensive discussions with Aravindan Vijayaraghavan throughout the course of its development. This work was supported in part by the National Science Foundation under grants ECCS-2216899 and ECCS-2216970.

\bibliographystyle{alpha}
\bibliography{refs}

\newpage
\clearpage
\appendix

\section{Additional related work}\label{app:related-work}

{\it Data-driven algorithm design} is a recently introduced paradigm for designing algorithms and provably tuning hyperparameters in machine learning~\citep{gupta2016pac,balcan2020data,sharma2024data}.  The framework can be viewed as a generalization of {\it average case analysis} from uniform distribution over the instances to arbitrary unknown distributions, that is, the tuned hyperparameters adapt to data distribution at hand.
Data-driven design has been successfully used for designing several fundamental learning algorithms including regression, low-rank approximation, tree search and many more (see e.g., ~\citep{balcan2017learning,enet_2022,balcan2024trees,balcan2024accelerating,balcan2025sample,blum2021learning,bartlett2022generalization,khodak2024learning,cheng2024learning,sakaue2024generalization,jinsample}).  
The techniques allow selection of near-optimal continuous hyperparameters, using multiple related tasks which are either drawn from an unknown distribution~\citep{balcan2017learning,balcan2021much} or arrive online~\citep{balcan2018dispersion,sharma2020learning,balcan2021learning,sharma2024no,sharma2025offline}. 
In fact tuning discretized parameters can lead to provably much worse performance than the best continuous parameter~\cite{balcan2024learning}. 

{\it Data-driven tuning of deep networks}. Recent work develops techniques for data-driven tuning of model hyperparameters in deep nets~\citep{balcan2025sample}, but their techniques do not apply to parameters of training algorithms including tuning the learning rate which is the focus of this work. At the technical level, our analysis apply to multiple hyperparameters (both \cite{gupta2016pac,balcan2025sample} apply only to single hyperparameters) and we do not need additional regularity assumptions (\cite{balcan2025sample} need the polynomials to be in a certain ``general position'') for our results. Finally, we tune hyperparameters like learning-rate schedules and pre-trained initializations, which must be tuned very carefully in practice, with provably bounded sample complexity.\looseness-1

{\it Gradient-descent analysis in non-convex regimes.} The non-convexity of neural network optimization problems is well-known~\citep{goodfellow2015qualitative}, although the theoretical understanding is quite limited. A large body of research had been devoted to understanding conditions under which gradient descent avoids saddle points and sub-optimal local minima~\citep{lee2016gradient,kawaguchi2016deep,jin2017escape,ge2017nospurious,daneshmand2018escaping}. Closer to the current work is the study of conditions under which gradient descent converges fast despite non-convexity~\citep{karimi2016linear,yue2023lowerpl}. Although interesting for showing linear convergence of gradient descent in a non-convex setting, these conditions are far from capturing actual optimization in deep learning. Our analysis handles realistic deep network optimization settings, and gives guarantees for learning the learning rates and schedules with provably optimal convergence rates.

{\it On tuning step-sizes and learning rate schedules.} Several heuristics for setting the step-sizes and learning rate schedules are known and regularly employed in practice~\citep{smith2015clr,loshchilov2016sgdr,goyal2017large}. This has motivated theoretical research into the impact of step-size and schedules on the convergence of gradient descent~\citep{ghadimi2013nonconvex,reddi2018convergence,oymak2021provable,kalra2024warmup}. While the impact is only understood in limited scenarios, it is observed more widely in practice. In the multi-task setting, several approaches are known including learning-to-optimize~\citep{Andrychowicz2016LearningTL}, meta-learning for SGD~\cite{Li2016LearningTO}, hypergradients~\citep{baydin2018online}, implicit differentiation~\citep{Maclaurin2015GradientbasedHO,pedregosa2016hypergrad,lorraine2020implicit}, and population-based training~\citep{jaderberg2017pbt}, to list a few. Another line of work shows the significance of adaptive stepsizes in gradient descent for logistic regression~\citep{wu2024large,zhang2025minimax}.

{\it Tuning initialization scale.} \cite{glorot2010understanding} is one of the first works to highlight the importance of carefully setting the initialization scale hyperparameter. While their approach works well for neural networks with sigmoid or tanh activations, the initialization proposed by \cite{he2015delving} is known to work better for ReLU and LeakyReLU activations. This highlights the need to set proper initialization scale, with strong dependence on the network architecture. Several efforts have been made to theoretically understand how the dynamics of the deep neural network are impacted by the initialization scale~\citep{saxe2013exact,poole2016exponential,pennington2017resurrecting,jacot2018neural,chizat2018lazy,zhang2019fixup}.

{\it Connections to pre-training.} Pre-training to learn foundation models forms the basis for successful language models~\citep{devlin2018bert,radford2018improving,liu2019linguistic,gururangan2020dont,wu2022limits}. Generalization guarantees and sample complexity are of fundamental interest towards a theoretical understanding of the effectiveness and data requirements of pre-trained models~\cite{yang2020generalization,tripuraneni2020theory}. Of particular interest for these models are properties like adversarial robustness and effectiveness under distribution shift~\cite{nern2022transferrobustness,wu2022power}, which are relevant future directions for this work. 

{\it Beyond gradient descent.} Our framework applies to gradient-based optimization beyond gradient descent. Interesting variants include second-order methods~\citep{nocedal2006numerical,henriques2018small}, Adam~\citep{kingma2014adam,zhuang2020adabelief}, Shampoo~\citep{gupta2018shampoo,lin2025understanding}, and many more~\citep{zhang2019lookahead,chen2021iterativekfac,pauloski2021kaisa}. In general, these methods are sensitive to hyperparameters which need careful tuning for good convergence and generalization~\citep{li2020rethinking,barakat2021understanding}. Understanding their convergence is an active area of research~\citep{reddi2018convergence,bock2018improvement,zhou2020convergence,defossez2020simple,dinh2020momentum,liu2021adamnc,bock2022nonconvergence}.

\section{GJ algorithm and pseudo-dimension }

A useful technique for bounding the pseudo-dimension in data-driven algorithm design is the GJ framework based approach proposed by \cite{bartlett2022generalization}. We include below the formal details for completeness.

\begin{definition}[\cite{goldberg1993bounding,bartlett2022generalization}]
    A GJ algorithm $\Gamma$ operates on real-valued inputs, and can perform two types of operations:
    \begin{itemize}
        \item Arithmetic operations of the form $v=v_0 \odot v_1$, where $\odot \in \{+, -, \times, \div\}$.
        \item Conditional statements of the form ``if $v \ge 0$ $\dots$ else $\dots$''.
    \end{itemize}
    In both cases, $v_0, v_1$ are either inputs or values previously computed by the algorithm (which are rational functions of the inputs). The {\it degree} of a GJ algorithm is the maximum degree of any rational function it computes of the inputs. The {\it predicate complexity} of a GJ algorithm is the number of distinct rational functions that appear in its conditional statements.
\end{definition}

\noindent The following theorem  due to \cite{bartlett2022generalization}  is useful in obtaining some of our pseudodimension bounds by showing a GJ algorithm that computes the loss for all values of the hyperparameters, on any fixed input instance.

\begin{theorem}[\cite{bartlett2022generalization}] \label{thm:gj}
    Suppose that each function $f \in \cF$ is specified by $n$ real parameters. Suppose that for every $x \in \cX$ and $r \in \R$, there is a GJ algorithm $\Gamma_{x, r}$ that given $f \in \cF$, returns ``true" if $f(x) \geq r$ and ``false" otherwise. Assume that $\Gamma_{x, r}$ has degree $\Delta$ and predicate complexity $\Lambda$. Then, $\mathrm{Pdim}(\cF) = O(n\log(\Delta\Lambda))$.
\end{theorem}

\section{Background on Real Zero Bounds}

Understanding the number of real solutions to systems of equations is a classical problem in real algebraic geometry. Two important results in this area are Warren's theorem for polynomials~\citep{warren1968lower} and its analogue for Pfaffian functions due to Khovanskii~\citep{khovanskiui1991fewnomials}.

\subsection{Warren's Theorem for Polynomials}

Let $f_1,\dots,f_s:\mathbb{R}^n\to\mathbb{R}$ be real polynomials of degree at most $d$. Warren's theorem \cite{warren1968lower} gives an upper bound on the number of connected components of the set
\[
\{x \in \mathbb{R}^n : f_i(x) > 0, \, i=1,\dots,s \}.
\]

\begin{theorem}[Warren, 1968]
\label{thm:warren}
Let $f_1,\dots,f_s$ be real polynomials in $n$ variables of degree at most $d$. Then the number of connected components of the set 
\[
\{ x\in \mathbb{R}^n : f_i(x) > 0 \text{ for all } i=1,\dots,s \}
\]
is at most
\[
\left( \frac{4ed s}{n} \right)^n.
\]
\end{theorem}

This result provides a polynomial analogue of the classical Descartes' rule of signs, giving a combinatorial bound on the complexity of semi-algebraic sets.

\subsection{Khovanskii's Theorem for Pfaffian Functions}

Pfaffian functions generalize polynomials and include many special functions such as exponentials and logarithms. Let $f_1,\dots,f_s:\mathbb{R}^n\to\mathbb{R}$ be Pfaffian functions of order $r$ and degree at most $\alpha$. Khovanskii \cite{khovanskiui1991fewnomials} extended Warren's type bounds to these functions.

\begin{lemma}[{\citealp[page~91]{khovanskiui1991fewnomials}}] \label{lm:kho-91}
Let $\cC$ be a Pfaffian chain of length $q$ and Pfaffian degree $M$, consists of functions $f_1,\dots, f_q$ in $\boldsymbol{a} \in \mathbb{R}^d$. Consider a non-singular system of equations $\Theta_1(\boldsymbol{a}) = \dots = \Theta_r(\boldsymbol{a}) = 0$ where $r \leq d$, in which $\Theta_i(\boldsymbol{a})$ ($i \in [r]$) is a polynomial of degree at most $\Delta$ in the variable $\boldsymbol{a}$ and in the Pfaffian functions $f_1, \dots, f_q$. Then the manifold of dimension $k = d - r$ determined by this system has at most $2^{q(q - 1)}\Delta^{ d}S^{d - r}[(r - d + 1)S- (r - d)]^q$ connected components, where $S = r(\Delta - 1) + dM + 1$. 
\end{lemma}
\noindent The following corollary is the direct consequence of Lemma \ref{lm:kho-91}. 

\begin{corollary} \label{cor:kho-91}
Consider the binary-valued function $\Phi(\boldsymbol{x}, \boldsymbol{a})$, for $\boldsymbol{x} \in \cX$ and $\boldsymbol{a} \in \mathbb{R}^d$ constructed using the boolean operators AND and OR, and boolean predicates in one of the two forms $``\tau(\boldsymbol{x}, \boldsymbol{a}) > 0"$ or $``\tau(\boldsymbol{x}, \boldsymbol{a}) = 0"$. Assume that the function $\tau(\boldsymbol{x}, \boldsymbol{a})$ can be one of at most $K$ forms ($\tau_1, \dots, \tau_K$), where $\tau_i(\boldsymbol{x}, \cdot)$ ($i \in [K])$ is a $C^\infty$ function of $\boldsymbol{a}$ for any fixed $x$. Assume that for any fixed $\boldsymbol{x}$, $\tau_i(\boldsymbol{x}, \cdot)$ ($i \in [K]$) is a Pfaffian function {of degree at most $\Delta$} from a Pfaffian chain $\cC$ with length $q$ and Pfaffian degree $M$. Then $$B \leq 2^{dq(dq - 1)/2}\Delta^{d}[(d^2(\Delta + M)]^{dq}.$$
\end{corollary}

Khovanskii's bound generalizes Warren's result from polynomials to a broad class of functions while keeping an explicit combinatorial control over the number of solutions.

As noted in the main body, Pfaffians include several commonly occurring functions including exponentials and logarithms. On the other hand, non-examples include some trigonometric functions like sine and cosine.

\section{Complete proofs}\label{app:proofs}
\thmgdpoly*

\begin{proof}
    We first claim that for $i\ge 2$, $x_i=(g_i^{(1)}(\eta),g_i^{(2)}(\eta),\dots,g_i^{(d)}(\eta))$ where $g_i^{(j)}$ is a polynomial function with degree at most $\Delta^{i-2}$. We will show this by induction.

    {\it Base case: $i=2$.} $x_2 = x_1 - \eta \nabla f(x_1) = x - \eta \nabla f(x)$ is a polynomial of degree $1=\Delta^{2-2}$ in $\eta$ in each coordinate. 

    {\it Inductive case: $i>2$.} Suppose $x_{i-1}=\boldsymbol{g}_{i-1}(\eta)=(g_{i-1}^{(j)}(\eta))_{j\in[d]}$ where $g_{i-1}^{(j)}$ is a polynomial of degree at most $\Delta^{i-3}$ (inductive hypothesis). Now $x_i=x_{i-1}-\eta \nabla f(x_{i-1})=\boldsymbol{g}_{i-1}(\eta)-\eta \nabla f(\boldsymbol{g}_{i-1}(\eta))=:\boldsymbol{g}_i(\eta)$. Clearly, $\boldsymbol{g}_i^{(j)}$ is a polynomial in $\eta$, with degree at most $\Delta^{i-3}(\Delta-1)+1 = \Delta^{i-2} -\Delta^{i-3}+1\le \Delta^{i-2}$.

    Thus, for any fixed $1\le i\le H$, $x_i=\boldsymbol{g}_i(\eta)$ where each coordinate $g_{i}^{(j)}$ is a polynomial with degree at most $\Delta^{i-2}$. For a fixed initial point $x$, this implies that $\nabla f(x_i)$ is a polynomial in $\eta$ of degree at most $\Delta^{i-1}$ in each coordinate. Thus, $\|\nabla f(x_i)\|^2$ is a polynomial of degree at most $2\Delta^{i-1}$ in $\eta$. Now, consider the set of points $\eta$ for which $\|\nabla f(x_i)\|^2<\theta^2$ for some constant $\theta$. This consists of at most $O(\Delta^{i-1})$ intervals. 

    Furthermore, we note that for any $\eta$, the cost is determined by the smallest $i$ such that $\|\nabla f(x_i)\|<\theta$ (if one exists). Since the cost takes only discrete values, it is a piecewise-constant function of $\eta$, and we seek to bound the number of these pieces. A naive counting argument gives a $O(\Pi_{i=1}^H\Delta^{i-1})=O(\Delta^{H^2})$ bound on the number of pieces, since if there are $K_{i-1}$ intervals corresponds to values of $\eta$ for which the algorithm converges within $i-1$ steps, then each of the $O(\Delta^{i-1})$ intervals  in round $i$ computed above may result in at most $K_{i-1}+1$ new pieces. We can, however, use an amortized counting argument to give a tighter bound. Indeed, suppose there are $K_{i-1}$ intervals with different values of cost $\le i-1$.  
    Of the new  $O(\Delta^{i-1})$ intervals say $T_i$ intersect at least one of the existing  pieces, resulting in at most  $T_i+K_{i-1}$ new pieces overall. 
    Thus, the total number of pieces of the cost function with cost $\le i$ is 
    $K_i\le (T_i+K_{i-1})+O(\Delta^{i-1})-T_i+K_{i-1}\le O(\Delta^{i-1})+2K_{i-1}$. Thus, across $i$,  
    we have at most $O(\sum_{i=1}^H2^{H-i}\Delta^{i-1})=O(\Delta^{H})$ intervals over each of which Algorithm \ref{alg:gd-vanilla} converges with a constant number of steps. The cost over the remainder of the domain $\cP$ where the algorithm does not converge, which consists of $O(\Delta^{H})$ intervals, is $H$.

    Thus, using Lemma \ref{lm:piecewise-constant-to-pdim}, since each function $\ell_{x,f}$ is $O(\Delta^{H})$-monotonic,  we get a bound on the pseudo-dimension of $\cL$ of $O(H\log\Delta)$, which implies the stated sample complexity.
\end{proof}

\thmmmpwpoly*
\begin{proof}
    We  claim that for $i\ge 2$, $$x_i=(g_i^{(1)}(\eta,\gamma),g_i^{(2)}(\eta,\gamma),\dots,g_i^{(d)}(\eta,\gamma))$$ and $$y_i=(h_i^{(1)}(\eta,\gamma),h_i^{(2)}(\eta,\gamma),\dots,h_i^{(d)}(\eta,\gamma)),$$ where each $g_i^{(j)}$ and $h_i^{(j)}$ is a piecewise-polynomial function with degree at most $\Delta^{i-2}$ and at most $(2pd)^{i-2}$ boundaries, each an algebraic curve with degree at most $\Delta^{i-2}$. We will show this by a simultaneous induction argument for $x_i$ and $y_i$.

    {\it Base case: $i=2$.} $y_2=-\eta \nabla f(x_1)$ and $x_2 = x_1+y_1= x_1 - \eta \nabla f(x_1)$ are both polynomials of degree $1=\Delta^{2-2}$ in $\eta,\gamma$ in each coordinate. 

    {\it Inductive case: $i>2$.} Suppose, by inductive hypothesis, that $x_{i-1}=\boldsymbol{g}_{i-1}(\eta)=(g_{i-1}^{(j)}(\eta))_{j\in[d]}$ and $y_{i-1}=\boldsymbol{h}_{i-1}(\eta)=(h_{i-1}^{(1)}(\eta,\gamma),h_{i-1}^{(2)}(\eta,\gamma),\dots,h_{i-1}^{(d)}(\eta,\gamma)),$ where $g_{i-1}^{(j)}$ and $h_{i-1}^{(j)}$ are piecewise-polynomial with at most $(2pd)^{i-3}$ pieces and degree at most $\Delta^{i-3}$ (for both pieces and boundaries).

    Now $g_i=\gamma g_{i-1}-\eta \nabla f(x_{i-1})=\gamma\boldsymbol{h}_{i-1}(\eta,\gamma)-\eta \nabla f(\boldsymbol{g}_{i-1}(\eta,\gamma))=:\boldsymbol{h}_i(\eta,\gamma)$. Since $\boldsymbol{g}_{i-1}(\eta,\gamma)$ is piecewise-polynomial, $\nabla f(\boldsymbol{g}_{i-1}(\eta,\gamma))$ is also piecewise-polynomial. The degree is at most $(\Delta-1)\Delta^{i-3}\le \Delta^{i-2}-1$. Any   boundary function (algebraic curve along which $\nabla f(\boldsymbol{g}_{i-1}(\eta,\gamma))$ is discontinuous) is either a boundary of $g_{i-1}^{(j)}$ for some $j$, or a solution of $f^{(k)}(g_{i-1}^{(1)}(\eta,\gamma),\dots,g_{i-1}^{(d)}(\eta,\gamma))=0$ for some boundary function $f^{(k)}$ ($k\in[p]$) of $f$. The total number of algebraic curves corresponding boundaries of $\boldsymbol{g}_i^{(j)}$ (for any $j\in[d]$) is therefore at most $d(2pd)^{i-3}+(pd)(2pd)^{i-3}\le (2pd)^{i-2}$, with degree at most $\Delta\cdot\Delta^{i-3}=\Delta^{i-2}$.
Therefore, $\boldsymbol{g}_i^{(j)}$ is piecewise-polynomial in $\eta,\gamma$, with degree at most $\Delta^{i-2}$ (for both pieces and boundaries) and at most $(2pd)^{i-2}$ boundaries.

    Thus, for any fixed $1\le i\le H$, $x_i=\boldsymbol{g}_i(\eta,\gamma)$ where each coordinate $g_{i}^{(j)}$ is piecewise-polynomial with degree at most $\Delta^{i-2}$ and at most $(2pd)^{i-2}$ boundaries. Using the argument above, for any fixed initial point $x$, $\|\nabla f(x_i)\|^2$ is piecewise-polynomial with degree at most $2\Delta^{i-1}$ in $\eta,\gamma$ and at most $(2pd)^{i-1}$ boundaries. 
    Therefore, for a fixed $\theta$, the condition $\|\nabla f(x_i)\|^2<\theta^2$ can be evaluated using a GJ algorithm~\citep{bartlett2022generalization} of degree $O(\Delta^i)$ and predicate complexity $O((2pd)^i)$. This implies that the dual cost function can be computed using a GJ algorithm with degree at most $O(\Delta^H)$ and predicate complexity $O(\sum_i(2pd)^i)=(2pd)^{O(H)}$. Using~\cite{bartlett2022generalization}, we get a bound on the pseudo-dimension of $\cL$ of $O(H\log pd\Delta)$, which implies the stated sample complexity.
\end{proof}

\thmpwpolyinitscale*

\begin{proof}
    The proof is similar to the proof of Theorem \ref{thm:gd-pwpoly-general-d}. Intuitively, we can express the $i$-th iterate as a polynomial of bounded degree in $\sigma$, much like $\eta$, and the rest of the argument applies. We include below the details for completeness.

    We first claim that for $i\ge 1$, $x_i=(g_i^{(1)}(\sigma),g_i^{(2)}(\sigma),\dots,g_i^{(d)}(\sigma))$ where each $g_i^{(j)}$ is a piecewise-polynomial function with degree at most $\Delta^{i-1}$ and at most $(2pd\Delta)^{i-1}$ pieces. We will show this by induction.

    {\it Base case: $i=1$.} $x_1 = \sigma \hat{x}$ is a polynomial of degree $1=\Delta^{1-1}$ in $\sigma$ in each coordinate. 

    {\it Inductive case: $i>1$.} Suppose $x_{i-1}=\boldsymbol{g}_{i-1}(\sigma)=(g_{i-1}^{(j)}(\sigma))_{j\in[d]}$ where $g_{i-1}^{(j)}$ is piecewise-polynomial with at most $(2pd\Delta)^{i-2}$ pieces and degree at most $\Delta^{i-2}$ (inductive hypothesis). Now $x_i=x_{i-1}-\eta \nabla f(x_{i-1})=\boldsymbol{g}_{i-1}(\sigma)-\eta \nabla f(\boldsymbol{g}_{i-1}(\sigma))=:\boldsymbol{g}_i(\sigma)$. $\nabla f(\boldsymbol{g}_{i-1}(\sigma))$ is piecewise-polynomial, with degree is at most $(\Delta-1)\Delta^{i-2}\le \Delta^{i-1}$. Any critical point is either a critical point of some $g_{i-1}^{(j)}$, or a solution of the equation $f^{(k)}(g_{i-1}^{(1)}(\sigma),\dots,g_{i-1}^{(d)}(\sigma))=0$ for some boundary function $f^{(k)}$ ($k\in[p]$) of $f$. The total number of critical points of $\boldsymbol{g}_i^{(j)}$ (for any $j\in[d]$) is therefore at most $(2pd\Delta)^{i-2}+(pd\Delta)(2pd\Delta)^{i-2}\le (2pd\Delta)^{i-1}$.
Therefore, $\boldsymbol{g}_i^{(j)}$ is piecewise-polynomial in $\sigma$, with degree at most $\Delta^{i-2}$ and at most $(2pd\Delta)^{i-1}$ pieces.

    Thus, for any fixed $1\le i\le H$, $x_i=\boldsymbol{g}_i(\sigma)$ where each coordinate $g_{i}^{(j)}$ is piecewise-polynomial with degree at most $\Delta^{i-1}$ and at most $(2pd\Delta)^{i-1}$ pieces. Using the argument above, for any fixed initial point $x$, $\|\nabla f(x_i)\|^2$ is piecewise-polynomial with degree at most $2\Delta^{i}$ in $\sigma$ and at most $(2pd\Delta)^{i}$ pieces. Now, consider the set of points $\sigma$ for which $\|\nabla f(x_i)\|^2<\theta^2$ for some constant $\theta$. This consists of at most $O((2pd\Delta)^{i})$ intervals.  
    
We can now apply the amortized counting argument from the proof of Theorem \ref{thm:gd-poly-general-d} to conclude that the number of pieces in the dual cost function is $O((2pd\Delta)^{H})$ here.
Using Lemma \ref{lm:piecewise-constant-to-pdim}, we get a bound on the pseudo-dimension of $\cL$ of $O(H\log pd\Delta)$, which implies the stated sample complexity.
\end{proof}

\thmpwpolyinit*

\begin{proof}
    
    Let $x=(x_0^{(1)},\dots,x_0^{(d)})$ denote the initial point hyperparameter. We first claim that for each $i\ge 1$, $x_i=(g_i^{(1)}(x_0^{(1)},\dots,x_0^{(d)}),\dots,g_i^{(d)}(x_0^{(1)},\dots,x_0^{(d)}))$ where each $g_i^{(j)}$ is a piecewise-polynomial function with degree at most $\Delta^{i-1}$ and at most $(2pd)^{i-1}$ algebraic boundaries of degree at most $\Delta^{i-1}$. We establish this by induction.

    {\it Base case: $i=1$.} $x_1 = (x_0^{(1)},\dots,x_0^{(d)}) $ is a polynomial of degree $1=\Delta^{1-1}$ in $x_0^{(1)},\dots,x_0^{(d)}$ in each coordinate. 

    {\it Inductive case: $i>1$.} Suppose $x_{i-1}=\boldsymbol{g}_{i-1}(x)=(g_{i-1}^{(j)}(x_0^{(1)},\dots,x_0^{(d)}))_{j\in[d]}$ where $g_{i-1}^{(j)}$ is piecewise-polynomial with at most $(2pd)^{i-2}$ boundaries and degree at most $\Delta^{i-2}$ (for both pieces and boundaries, by the inductive hypothesis). 
    
    Now $x_i=x_{i-1}-\eta \nabla f(x_{i-1})=\boldsymbol{g}_{i-1}(x)-\eta \nabla f(\boldsymbol{g}_{i-1}(x))=:\boldsymbol{g}_i(x)$. Now, $\nabla f(\boldsymbol{g}_{i-1}(x))$ is piecewise-polynomial. The degree is at most $(\Delta-1)\Delta^{i-2}\le \Delta^{i-1}$. Any discontinuity point is either a discontinuity point of some $g_{i-1}^{(j)}$, or a solution of the equation $f^{(k)}(g_{i-1}^{(1)}(x),\dots,g_{i-1}^{(d)}(x))=0$ for some boundary function $f^{(k)}$ ($k\in[p]$) of $f$. The total number of boundary functions of $\boldsymbol{g}_i^{(j)}$ (for any $j\in[d]$) is therefore at most $d(2pd)^{i-2}+(pd)(2pd)^{i-2}\le (2pd)^{i-1}$.
Therefore, $\boldsymbol{g}_i^{(j)}$ is piecewise-polynomial in $x$, with degree at most $\Delta^{i-1}$ and at most $(2pd\Delta)^{i-1}$ pieces.

    Thus, for any fixed $1\le i\le H$, $x_i=\boldsymbol{g}_i(x)$ where each coordinate $g_{i}^{(j)}$ is piecewise-polynomial with degree at most $\Delta^{i-1}$ and at most $(2pd)^{i-1}$ algebraic boundaries. Using the argument above, for any fixed $\eta$, $\|\nabla f(x_i)\|^2$ is piecewise-polynomial with degree at most $2\Delta^{i}$ in the initial vector $x$ and at most $(2pd\Delta)^{i}$ boundaries. This implies that the dual cost function can be computed using a GJ algorithm with degree at most $O(\Delta^H)$ and predicate complexity $O(\sum_i(2pd)^i)=(2pd)^{O(H)}$. Using~\cite{bartlett2022generalization}, we get a bound on the pseudo-dimension of $\cL$ of $O(dH\log pd\Delta)$, which implies the stated sample complexity.
\end{proof}

\thmvalidation*
\begin{proof}
    As shown in the proof of Theorem \ref{thm:gd-pwpoly-general-d}, $\tilde{x}=g(\eta)$ where $g$ is a piecewise polynomial function with degree at most $\Delta^H$ and at most $O((2pd\Delta)^{H})$ pieces. Therefore, $f_v(\tilde{x})=f_v(g(\eta))$ is a piecewise polynomial function with degree at most $\Delta_v\Delta^H$ and at most $p_v+(2pd\Delta)^{H}$ pieces. Using~\cite{bartlett2022generalization}, we get a bound on the pseudo-dimension of loss function class of $O(H\log (pd\Delta)+\log (\Delta_vp_v))$, which implies the stated sample complexity.
\end{proof}

\end{document}